\documentclass{article}

     \usepackage{natbib}\date{\empty}
\usepackage[utf8]{inputenc} 
\usepackage[T1]{fontenc}    
\usepackage{xcolor}
\definecolor{mydarkred}{rgb}{0.6,0,0}
\definecolor{mydarkgreen}{rgb}{0,0.6,0}
\usepackage[colorlinks,
linkcolor=mydarkred,
citecolor=mydarkgreen]{hyperref}
\usepackage{url}            
\usepackage{booktabs}       
\usepackage{amsfonts}       
\usepackage{nicefrac}       
\usepackage{microtype}      
\usepackage{bm,bbm}
\usepackage{amsmath}
\usepackage{amssymb}
\usepackage{amsthm}
\usepackage{mathtools}
\usepackage{comment}
\usepackage{url}
  
\usepackage[pdftex]{graphicx}

\newtheorem{thm}{Theorem}
\newtheorem{lem}{Lemma}
\setcitestyle{numbers,open={[},close={]}}
\usepackage{symbols}
\usepackage{color}

\title{Uncoupled Regression\\ from Pairwise Comparison Data}

\author{
Liyuan Xu\textsuperscript{1,2} \thanks{liyuan@ms.k.u-tokyo.ac.jp}
\and
Junya Honda\textsuperscript{1,2} \thanks{honda@stat.t.u-tokyo.ac.jp}
\and
Gang Niu\textsuperscript{2} \thanks{gang.niu@riken.jp}
\and
Masashi Sugiyama\textsuperscript{2,1}\thanks{sugi@k.u-tokyo.ac.jp}\vspace{3mm}\\
\centerline{\textsuperscript{1}The University of Tokyo, \textsuperscript{2}RIKEN}
}

\begin{document}
\maketitle
\begin{abstract}
  \emph{Uncoupled regression} is the problem to learn a model from unlabeled data and the set of target values while the correspondence between them is unknown. Such a situation arises in predicting anonymized targets that involve sensitive information, e.g., one's annual income. Since existing methods for uncoupled regression often require strong assumptions on the true target function, and thus, their range of applications is limited, we introduce a novel framework that does not require such assumptions in this paper. Our key idea is to utilize \emph{pairwise comparison data}, which consists of pairs of unlabeled data that we know which one has a larger target value. Such pairwise comparison data is easy to collect, as typically discussed in the learning-to-rank scenario, and does not break the anonymity of data. We propose two practical methods for uncoupled regression from pairwise comparison data and show that the learned regression model converges to the optimal model with the optimal parametric convergence rate when the target variable distributes uniformly. Moreover, we empirically show that for linear models the proposed methods are comparable to ordinary supervised regression with labeled data.
\end{abstract}

\section{Introduction} \label{sec:introduction}
In supervised regression, we need a vast amount of labeled data in the training phase, which is costly and laborious to collect in many real-world applications. To deal with this problem, weakly-supervised regression has been proposed in various settings, such as semi-supervised learning (see \citet{Kostopoulos2018} for the survey), 
multiple instance regression \citep{Ray2001,Zhang2002}, and
transductive regression \citep{Cortes2008,Cortes2008-2}. 
See \citep{Zhou2017}
for thorough review for the weakly-supervised learning in binary classification, which can be extended to regression with slight modifications.

Uncoupled regression \citep{Alexandra2016}
is one variant of weakly-supervised learning. In ordinary ``coupled'' regression, the pairs of features and targets are provided, and we aim to learn a model which minimizes the prediction error in test data. On the other hand, in the uncoupled regression problem, we only have access to unlabeled data and the set of target values, and we do not know the true target for each data point. Such a situation often arises when we aim to predict people’s sensitive matters such as one's annual salary or total amount of deposit, the data of which is often anonymized for privacy concerns. Note that it is impossible to conduct uncoupled regression without further assumptions, since no labeled data is provided. 

\citet{Alexandra2016} showed that uncoupled regression is solvable if the true target function is monotonic to a single dimensional feature by matching the empirical distributions of the feature and the target. Although their algorithm is of less practical use due to its strong assumption, their work offers a valuable insight, which is that a model is learnable from uncoupled data if we know the ranking in the dataset. In this paper, we show that, instead of imposing the monotonic assumption, we can infer such ranking information from data to solve uncoupled regression. We use \emph{pairwise comparison data} as a source of ranking information, which consists of the pairs of unlabeled data that we know which data point has a larger target value. 

Note that pairwise comparison data is easy to collect even for sensitive matters such as one's annual earnings. Although people often hesitate to give explicit answers of it, it might be easier to answer indirect questions: ``Which person earns more than you?'' \footnote{This questioning can be regarded as one type of randomized
response (indirect questioning) techniques \citep{Warner1965}, which is a survey method to avoid social desirability bias.}, which yields pairwise comparison data that we needed. Considering that we do not put any assumption on the true target function, our method is applicable to many situations.


One naive method for uncoupled regression with pairwise comparison data is to use a score-based ranking method \citep{Rudin2005}, which learns a score function with the minimum inversions in pairwise comparison data. With such a score function, we can match unlabeled data and the set of target values, and then, conduct supervised learning. However, as discussed in \citet{Rigollet2019}, we cannot consistently recover the true target function even if we know the true order of unlabeled data, when the target variable contains noise. 

In contrast, our method directly minimizes the regression risk. We first rewrite the regression risk so that it can be estimated from unlabeled and pairwise comparison data, and learn a model through empirical risk minimization. Such an approach based on risk rewriting has been extensively studied in the classification scenario \citep{duPlessis2014,duPlessis2015,Gang2016,Sakai2017,lu2018} and exhibits promising performance. We consider two estimators of the risk defined based on the expected Bregman divergence \citep{Frigyik2008}, which is a natural choice of the risk function. We show that if the target variable is marginally distributed uniformly then the estimators are unbiased and the learned model converges to the optimal model with the optimal rate in such a case. In general cases, however, we prove that it is impossible to have such an unbiased estimator in any marginal distributions and the learned model may not converge to the optimal one. Still, our empirical evaluations based on synthetic data and benchmark datasets show that our methods exhibit similar performances as a model learned from ordinary supervised learning.


The paper is structured as follows. After discussing the related work in Section~\ref{sec:related_work}, we formulate the uncoupled regression problem with pairwise comparison data in detail in Section~\ref{sec:problem_settings}. In Sections~\ref{sec:linear-approx} and~\ref{sec:cdf-bregman}, we discuss two methods for uncoupled regression and derive estimation error bounds for each method. Finally, we show empirical results in Section~\ref{sec:experiments} and conclude the paper in Section~\ref{sec:conclusion}.

\section{Related Work} \label{sec:related_work}

Several methods have been proposed to match two independently collected data sources. In the context of data integration \citep{Cohen2002}, the matching is conducted based on some contextual data provided for both data sources. For example, \citet{Walter1999} used spatial information as contextual data to integrate two data sources. Some work evaluated the quality of matching by some information criterion and found the best matching by the maximization of the metrics. This problem is called cross-domain object matching (CDOM), which is formulated in \citet{Jebara2004}. A number of methods have been proposed for CDOM, such as \citet{Quadrianto2008,Yamada2011,Aditya2017}. 

Another line of related work in the uncoupled regression problem imposed an assumption on the true target function.
For example, \citet{Alexandra2016} assumed that the true target function
is monotonic to a single feature, and it was refined by \citet{Rigollet2019}. Another common assumption is that the true target function is a linear function of the features,
which was studied in \citet{Hsu2017} and \citet{Pananjady2018}. Although these methods yield accurate models, they are of less practical use due to their strong assumptions. On the other hand, our methods do not require any assumptions on such mapping functions and are applicable to wider scenarios.

It is worth noting that some methods use uncoupled data to enhance the performance of semi-supervised learning. For example, in label regularization \citep{Mann2010}, uncoupled data is used to regularize a regression model so that the distribution of prediction on unlabeled data is close to the marginal distribution of target variables, which is reported to increase the accuracy.


Pairwise comparison data was originally considered in the ranking problem \citep{Rudin2005,Mohri2012}, which aims to learn a score function that can rank data correctly. In fact, we can apply ranking methods, such as rankSVM \citep{Herbrich2000}, to our problem. However, the naive application of them performs inferiorly compared to proposed methods, as we will show empirically, since our goal is not to order data correctly but to predict true target values. 

\section{Problem Settings} \label{sec:problem_settings}
In this section, we formulate the uncoupled regression problem and introduce pairwise comparison data. We first define the uncoupled regression, and then, we describe the data generating process of pairwise comparison data.

\subsection{Uncoupled Regression Problem}

We first formulate the standard regression problem briefly. Let $\mathcal{X} \subset \mathbb{R}^d$ be a $d$-dimensional feature space and $\mathcal{Y} \subset \mathbb{R}$ be a target space. We denote $\vec{X},Y$ as random variables on spaces $\mathcal{X},\mathcal{Y}$, respectively. We assume these random variables follow the joint distribution $P_{\vec{X},Y}$. The goal of the regression problem is to obtain model $h:\mathcal{X} \to \mathcal{Y}$ in hypothesis space $\mathcal{H}$
which minimizes the risk defined as
\begin{align}
R(h) = \expect{\vec{X},Y}{l(h(\vec{X}),Y)}, \label{eq:expected_loss}
\end{align}
where $\mathbb{E}_{\vec{X},Y}$ denotes the expectation over $P_{\vec{X},Y}$ and $l:\mathcal{Y}\times \mathcal{Y} \to \mathbb{R}_+$ is a loss function.

The loss function $l(z, t)$ measures the closeness between a true target $t \in \mathcal{Y}$ and an output of a model $z \in \mathcal{Y}$, which generally grows as the prediction $z$ gets far from the target $t$. In this paper, we mainly consider $l(z,t)$ to be the Bregman divergence $d_\phi(t,z)$, which is defined as
\begin{align*}
    d_\phi(t,z) = \phi(t) - \phi(z) - (t-z)\phi'(z)
\end{align*}
for some convex function $\phi:\mathbb{R} \to \mathbb{R}$, and $\phi'$ denotes the derivative of $\phi$. 
It is natural to have such a loss function since the minimizer of risk $R$ is $\expect{Y|\vec{X}=\vec{x}}{Y}$ when hypothesis space $\mathcal{H}$ is rich enough \citep{Frigyik2008}, where $\mathbb{E}_{Y|\vec{X}=\vec{x}}$ is the conditional expectation over the distribution of $Y$ given $\vec{X}=\vec{x}$.
Many common loss functions can be interpreted as the Bregman divergence; for instance, when $\phi(x) = x^2$, then $d_\phi(t,z)$ becomes the $l_2$-loss, and when $\phi(x) = x\log x - (1-x)\log(1-x)$, then $d_\phi(t,z)$ becomes the Kullback–Leibler divergence between the Bernoulli distributions with probabilities $t$ and $z$.

In the standard regression scenario, we are given labeled training data $\mathcal{D} = \{(\vec{x}_i, y_i)\}_{i=1}^n$ drawn independently and identically from $P_{\vec{X}, Y}$. Then, based on the training data, we empirically estimate risk $R(h)$ 
and learn model $\hat{h}$ as the minimizer of the empirical risk. However, in uncoupled regression, no individual label is available, and thus this approach is no longer applicable. 
Instead of ordinary ``coupled'' data, what we are given is unlabeled data $\mathcal{D}_\unlabeled = \{\vec{x}_i\}^{n_\unlabeled}_{i=1}$ and target values $\mathcal{D}_Y = \{y_i\}_{i=1}^{n_Y}$. Here, $n_\unlabeled$ is the size of unlabeled data. Furthermore, we denote the marginal distribution of feature $\vec{X}$ as $P_{\vec{X}}$ and its probability density function as $f_{\vec{X}}$. Similarly, $P_{Y}$ stands for the marginal distribution of target $Y$, and $f_Y$ is the density function of $P_Y$. We use $\mathbb{E}_{\vec{X},Y}, \mathbb{E}_{\vec{X}}$ and $\mathbb{E}_{Y}$ to denote the expectation over $P_{\vec{X},Y}, P_{\vec{X}}$, and $P_Y$, respectively.

Unlike \citet{Alexandra2016}, we do not try to match unlabeled data and target values. In fact, our methods do not use each target value in $\mathcal{D}_Y$ but use density function $f_Y$ of the target, which can be estimated from $\mathcal{D}_Y$. For simplicity, we assume that the true density function $f_Y$ is known. The case where we need to estimate $f_Y$ from $\mathcal{D}_Y$ is discussed in Appendix~\ref{sec:est-prob}.

\subsection{Pairwise Comparison Data}

Here, we introduce pairwise comparison data. It consists of two random variables $(\vec{X}^+, \vec{X}^-)$, where the target value of $\vec{X}^+$ is larger than that of $\vec{X}^-$. Formally, $(\vec{X}^+, \vec{X}^-)$ are defined as 
\begin{align}
    \vec{X}^+ = \begin{cases}
    \vec{X} & (Y \geq Y'),\\
    \vec{X}' & (Y < Y'),\\
    \end{cases}
    \quad 
    \vec{X}^- = \begin{cases}
    \vec{X}' & (Y \geq Y'),\\
    \vec{X} & (Y < Y'),\\
    \end{cases} \label{eq:def_X_compare}
\end{align}
where $(\vec{X}, Y), (\vec{X}', Y')$ are two independent random variables following $P_{\vec{X},Y}$. We denote the joint distribution of $(\vec{X}^+,\vec{X}^-)$ as $P_{\vec{X}^+,\vec{X}^-}$ and the marginal distributions as $P_{\vec{X}^+},P_{\vec{X}^-}$. Density functions $f_{\vec{X}^+,\vec{X}^-}, f_{\vec{X}^+}, f_{\vec{X}^-}$ and expectations $\mathbb{E}_{\vec{X}^+,\vec{X}^-}, \mathbb{E}_{\vec{X}^+}, \mathbb{E}_{\vec{X}^-}$ are defined in the same way.

We assume that we have access to $n_\rank$ pairs of i.i.d. samples of $(\vec{X}^+, \vec{X}^-)$ as $\mathcal{D}_\rank = \{(\vec{x}^+_i,\vec{x}^-_i)\}_{i=1}^{n_\rank}$ in addition to unlabeled data $\mathcal{D}_\unlabeled$ and  density function $f_Y$ of target variable $Y$.
In the following sections, we show that uncoupled regression can be solved only from this information. In fact, our method only requires samples of \emph{either one of} $\vec{X}^+, \vec{X}^-$, which corresponds to the case where only a winner or loser of the ranking is observable.

One naive approach to conduct uncouple regression with $\mathcal{D}_\rank$ would be to adopt \emph{ranking methods}, which is to learn a ranker $r: \mathcal{X} \to \mathbb{R}$ that minimizes the following expected ranking loss:
\begin{align}
    R_{\rank}(r) = \expect{\vec{X}^+,\vec{X}^-}{\indi{r(\vec{X}^+)-r(\vec{X}^-)<0}}, \label{eq:ranking_ross}
\end{align}
where $\mathbbm{1}$ is the indicator function. By minimizing the empirical estimation of \eqref{eq:ranking_ross} based on $\mathcal{D}_\rank$, we can learn a ranker $\hat{r}$ that can sort data points by target $Y$. Then, we can predict quantiles of test data by ranking $\mathcal{D}_\unlabeled$, which leads to the prediction by applying the inverse of the cumulative distribution function (CDF) of $Y$. Formally,
if the test point $\vec{x}_\mathrm{test}$ is ranked top $n'$-th in $\mathcal{D}_\unlabeled$, we can predict the target value for $\vec{x}_\mathrm{test}$ as
\begin{align}
    \hat{h}(\vec{x}_\mathrm{test}) = F^{-1}_Y\left(\frac{n_\unlabeled - n'}{n_\unlabeled}\right), \label{eq:naive-regression-function}
\end{align}
where $F_Y(t) = P(Y \leq t)$ is the CDF of $Y$. 

This approach, however, is known to be highly sensitive to the noise as discussed in \citet{Rigollet2019}. This is because a noise involved in the single data point changes the ranking of all other data points and affects their predictions. As illustrated in \citet{Rigollet2019}, even if when we have a perfect ranker, i.e., we know the true order in $\mathcal{D}_\unlabeled$, model \eqref{eq:naive-regression-function} is still different from the expected target $Y$ given feature $\vec{X}$ in presence of noise.

\section{Empirical Risk Minimization by Risk Approximation}  \label{sec:linear-approx}

In this section, we propose a method to learn a model from pairwise comparison data $\mathcal{D}_\rank$, unlabeled data $\mathcal{D}_\unlabeled$, and density function $f_Y$ of target variable $Y$. The method follows the empirical risk minimization principle, while the risk is approximated so that it can be empirically estimated from data available. Therefore, we call this approach as \textit{risk approximation} (RA) approach. Here, we present an approximated risk and derive its estimation error bound. 


From the definition of the Bregman divergence, the risk function in \eqref{eq:expected_loss} is expressed as
\begin{align}
    R(h) = \expect{Y}{\phi(Y)} - \expect{\vec{X}}{\phi(h(\vec{X})) - h(\vec{X})\phi'(h(\vec{X}))} -  \expect{\vec{X},Y}{Y\phi'(h(\vec{X}))}. \label{eq:bregman-divergence-trivial-decompose}
\end{align}
In this decomposition, the last term is the only problematic part in uncoupled regression since it requires to calculate the expectation on the joint distribution.
Here, we consider approximating the last term based on the following expectations over the distributions of $\vec{X}^+,\vec{X}^-$

\begin{lem}\label{lem:X_compare_distribution}
    We have
    \begin{align*}
        &\expect{\vec{X}^+}{\phi'(h(\vec{X}^+))} = 2\expect{\vec{X},Y}{F_Y(Y)\phi'(h(\vec{X}))},\\
        &\expect{\vec{X}^-}{\phi'(h(\vec{X}^-))} = 2\expect{\vec{X},Y}{(1-F_Y(Y))\phi'(h(\vec{X}))}.
    \end{align*}
\end{lem}
The proof can be found in Appendix~\ref{sec:proof-of-X_compare_distribution}. From Lemma~\ref{lem:X_compare_distribution}, we can see that $\expect{\vec{X},Y}{Y\phi'(h(\vec{X}))} = (\expect{\vec{X}^+}{\phi'(h(\vec{X}^+))})/2$ if $F_Y(y) = y$, which corresponds to the case that target variable $Y$ marginally distributes uniformly in $[0,1]$. This motivates us to consider the approximation in the form of
\begin{align}
    \expect{\vec{X},Y}{Y\phi'(h(\vec{X}))} \simeq  w_1 \expect{\vec{X}^+}{\phi'(h(\vec{X}^+))} + w_2 \expect{\vec{X}^-}{\phi'(h(\vec{X}^-))} \label{eq:linear-approximation}
\end{align}
for some constants $w_1,\,w_2 \in \mathbb{R}$. Note that the above uniform case corresponds to $(w_1,w_2)=(1/2,0)$. In general, if target $Y$ marginally distributes uniformly on $[a,b]$ for $b > a$, that is, $F_Y(y) = (y-a)/(b-a)$ for all $y\in[a,b]$, we can see that approximation \eqref{eq:linear-approximation} becomes exact for $(w_1,w_2) =(b/2,a/2)$ from Lemma~\ref{lem:X_compare_distribution}. In such a case, we can construct an unbiased estimator of true risk $R$ from unlabeled and pairwise comparison data. For non-uniform target marginal distributions, we choose $(w_1,w_2)$ that minimizes the upper bound of the estimation error, which we will discuss in detail later.

Since we have $\expect{\vec{X}}{\phi'(\vec{X})} = \frac12\expect{\vec{X}^+}{\phi'(\vec{X}^+)} + \frac12\expect{\vec{X}^-}{\phi'(\vec{X}^-)}$ from Lemma~\ref{lem:X_compare_distribution},  the RHS of \eqref{eq:linear-approximation} equals
\begin{align}
    \lambda \expect{\vec{X}}{\phi(\vec{X})} + \left(w_1 - \frac{\lambda}2\right) \expect{\vec{X}^+}{\phi'(h(\vec{X}^+))} + \left(w_2 - \frac{\lambda}2\right) \expect{\vec{X}^-}{\phi'(h(\vec{X}^-))} \label{eq:linear-approximation-2}
\end{align}
for arbitrary $\lambda \in \mathbb{R}$. Hence, by approximating \eqref{eq:bregman-divergence-trivial-decompose} by \eqref{eq:linear-approximation-2},
we can write the approximated risk $R_{\linear}$ as
\begin{align*}
    R_{\linear}(h; \lambda, w_1, w_2) &= \mathfrak{C} - \expect{\vec{X}}{\phi(h(\vec{X})) -  (h(\vec{X}) - \lambda)\phi'(h(\vec{X}))}\\
    &\quad~~~ -  \left(w_1 - \frac{\lambda}2\right) \expect{\vec{X}^+}{\phi'(h(\vec{X}^+))} - \left(w_2 - \frac{\lambda}2\right) \expect{\vec{X}^-}{\phi'(h(\vec{X}^-))},
\shortintertext{Here, $\mathfrak{C} = \expect{Y}{\phi(Y)}$ can be ignored in the optimization process. Now, the empirical estimator of $R_{\linear}$ is}
    \hat{R}_\linear(h; \lambda, w_1, w_2) &= \mathfrak{C} - \frac{1}{n_\unlabeled} \sum_{\vec{x}_i \in \mathcal{D}_\unlabeled} \left(\phi(h(\vec{x}_i)) - (h(\vec{x}_i) - \lambda)\phi'(h(\vec{x}_i)) \right)\\
    &\quad~~~ -  \frac{1}{n_\rank} \sum_{(\vec{x}^+_i, \vec{x}^-_i) \in \mathcal{D}_\rank}  \left(\left(w_1 - \frac{\lambda}2\right)\phi'(h(\vec{x}_i^+)) + \left(w_2 - \frac{\lambda}2\right) \phi'(h(\vec{x}_i^-))\right), 
\end{align*}
which is to be minimized in the RA approach. Again, we would like to emphasize that if marginal distribution $P_Y$ is uniform on $[a,b]$ and $(w_1,w_2)$ is set to $(b/2, a/2)$, we have $R_\linear = R$ and $\hat{R}_\linear$ is an unbiased estimator of $R$.

From the definition of $\hat{R}_\linear$, we can see that by setting $\lambda$ to either $2w_1$ or $2w_2$,  $\hat{R}_\linear$ becomes independent of either $\vec{X}^+$ or $\vec{X}^-$. This means that we can conduct uncouple regression even if one of $\vec{X}^+, \vec{X}^-$ is missing in data, which corresponds to the case where only winners or only losers of the comparison are observed. 

Another advantage of tuning free parameter $\lambda$ is that we can reduce the variance in empirical risk $\hat{R}_\linear$ as discussed in \citet{Sakai2017} and \citet{Bao2018}. As in \citet{Sakai2017}, the optimal $\lambda$ that minimizes the variance in $\hat{R}_\linear$
for $n_\unlabeled \to \infty$ is derived as follows. 
\begin{thm}\label{thm:optimal-lambda-in-uniform}
For given model $h$, let $\sigma^2_+, \sigma^2_-$ be 
\begin{align*}
    \sigma^2_+ = \Var{\vec{X}^+}{\phi'(h(\vec{X}^+))}, \quad \sigma^2_- = \Var{\vec{X}^-}{\phi'(h(\vec{X}^-))}, 
\end{align*}
respectively, where $\Var{\vec{X}}{\cdot}$ is the variance with respect to the random variable $\vec{X}$. Then, setting
\begin{align*}
    \lambda = \frac{2(w_1\sigma^2_+ +w_2\sigma^2_-)}{\sigma^2_+ + \sigma^2_-} \vspace{-25pt}
\end{align*}
yields the estimator with the minimum variance among estimators in the form of $\hat{R}_\linear$ when $n_\unlabeled \to \infty$.
\end{thm}
The proof can be found in Appendix~\ref{sec:proof-of-optimal-lambda-in-uniform}. From Theorem~\ref{thm:optimal-lambda-in-uniform}, we can see that the optimal $\lambda$ does not equal zero, which means we can reduce the variance in the empirical estimation with a sufficient number of unlabeled data by tuning $\lambda$. Note that this situation is natural since unlabeled data is easier to collect than pairwise comparison data as discussed in \citet{Duh2011}.

Now, from the discussion of the the pseudo-dimension \citep{Haussler1992}, we establish the upper bound of the estimation error, which is used to choose weights $(w_1,w_2)$. Let $\hat{h}_\linear, h^*$ be the minimizers of $\hat{R}_\linear$ and $R$ in hypothesis class $\mathcal{H}$, respectively. Then, we have the following theorem that bounds the excess risk in terms of parameters $(w_1,w_2)$.
\begin{thm}\label{thm:uniform-taylor-bound}
    Suppose that the pseudo-dimensions of $\{\vec{x} \to \phi'(h(\vec{x}))~|~h \in \mathcal{H}\}, \{\vec{x} \to h(\vec{x})\phi'(h(\vec{x})) - \phi(h(\vec{x}))~|~h \in \mathcal{H}\}$ are finite and there exist constants $m,M$ such that $|h(\vec{x})\phi'(h(\vec{x})) - \phi(h(\vec{x}))| \leq m, |\phi'(h(\vec{x}))| \leq M$ for all $\vec{x} \in \mathcal{X}$ and all $h \in \mathcal{H}$. Then,
\begin{align*}
    R(\hat{h}_\linear) \leq R(h^*) + O\left(\sqrt{\frac{\log 1/\delta}{n_\unlabeled}}\right) + O\left(\sqrt{\frac{\log 1/\delta}{n_\rank}}\right) + M\mathrm{Err}(w_1,w_2)
\end{align*}
holds with probability $1-\delta$, where $\mathrm{Err}$ is defined as
\begin{align}
    \mathrm{Err}(w_1,w_2) = \expect{Y}{\left|Y - 2w_1F_Y(Y) - 2w_2(1-F_Y(Y))\right|}.  \label{eq:err-def} \vspace{-30pt}
\end{align}
\end{thm}

The proof can be found in Appendix~\ref{sec:proof-of-uniform-taylor-bound}. Note that the conditions of boundedness of $|h(\vec{x})\phi'(h(\vec{x})) - \phi(h(\vec{x}))|,|\phi'(h(\vec{x}))|$ hold for many losses, e.g., $l_2$-loss, when we consider a hypothesis space of bounded functions.

From Theorem~\ref{thm:uniform-taylor-bound}, we can see that we can learn a model with less excess risk by minimizing $\mathrm{Err}(w_1,w_2)$. Note that $\mathrm{Err}(w_1,w_2)$ can be easily minimized since density function $f_Y$ is known or can be estimated from $\mathcal{D}_Y$. In particular, if target $Y$ is uniformly distributed on $[a,b]$, we have $\mathrm{Err}(w_1,w_2) = 0$ by setting $(w_1,w_2) =(b/2,a/2)$. In such a case, $\hat{h}_\linear$ becomes a consistent model, i.e., $R(\hat{h}_\linear) \to R(h^*)$ as $n_\unlabeled \to \infty$ and $n_\rank \to \infty$. The convergence rate is $O(1/\sqrt{n_\unlabeled}+1/\sqrt{n_\rank})$, which is optimal parametric rate for
the empirical risk minimization without additional assumptions when we have enough amount of unlabeled and pairwise comparison data jointly \citep{Mendelson2008}.

One important case where target variable $Y$ distributes uniformly is when the target is  ``quantile value''. For instance, we are to build a screening system for credit cards. Then, what we are interested in is ``how much an applicant is credible in the population?'', which means that we want to predict the quantile value of the ``credit score'' in the marginal distribution. By definition, we know that such a quantile value distributes uniformly, and thus we can have a consistent model by minimizing $\hat{R}_\linear$.

In general cases, however, we may have $\mathrm{Err}(w_1,w_2) > 0$, and $\hat{h}_\linear$ becomes not consistent. Nevertheless, this is inevitable as suggested in the following theorem.
\begin{thm}\label{thm:impossibility}
    There exists a pair of joint distributions $P_{\vec{X},Y}, \tilde{P}_{\vec{X},Y}$ that yields the same marginal distributions of feature $P_{\vec{X}}$ and target $P_Y$, and the same distributions of the pairwise comparison data $P_{\vec{X}^+,\vec{X}^-}$ but have different conditional expectation $\expect{Y | \vec{X}=\vec{x}}{Y}$.
\end{thm}
Theorem~\ref{thm:impossibility} states that there exists a pair of distributions that cannot be distinguished from available data. Considering that $h^*(\vec{x}) = \expect{Y|\vec{X} = \vec{x}}{Y}$ when hypothesis space $\mathcal{H}$ is rich enough \citep{Frigyik2008}, this theorem implies that we cannot always obtain a consistent model. Still, by tuning weights $(w_1,w_2)$, we can obtain a  model competitive with the consistent one. In Section~\ref{sec:experiments}, we show that $h_\linear$ empirically exhibits a similar accuracy to a model learned from ordinary coupled data.

\section{Empirical Risk Minimization by Target Transformation} \label{sec:cdf-bregman}

In this section, we introduce another approach to uncoupled regression with pairwise comparison data, called the \textit{target transformation} (TT) approach. Whereas the RA approach minimizes the approximation of the original risk, the TT approach transforms the target variable so that it marginally distributes uniformly, and it minimizes an unbiased estimator of the risk defined based on the transformed variable.

Although there are several ways to map $Y$ to a uniformly distributed random variable, one natural candidate  would be CDF $F_Y(Y)$, which leads to considering the following risk:
\begin{align}
    R_{\CDF}(h) &= \expect{\vec{X},Y}{d_\phi(F_Y(Y),F_Y(h(\vec{X}))}. \label{eq:bregman-CDF-decomposition}
\end{align}
Since $F_Y(Y)$ distributes uniformly on $[0,1]$ by definition, we can construct the following unbiased estimator of $R_\CDF$ below from the same discussion as in the previous section.
\begin{align*}
    \hat{R}_{\CDF}(h; \lambda) &= \mathfrak{C} - \frac{1}{n_\unlabeled} \sum_{\vec{x}_i \in \mathcal{D}_\unlabeled} \big((\lambda -F_Y(h(\vec{x}_i)))\phi'(F_Y(h(\vec{x}_i))) + \phi(F_Y(h(\vec{x}_i)))\big) \notag\\
    &\quad~~~~~  - \frac{1}{n_\rank}\sum_{(\vec{x}^+_i, \vec{x}^-_i)\in \mathcal{D}_\rank} \left(\frac{1-\lambda}{2}\phi'(F_Y(h(\vec{x}_i^+))) - \frac{\lambda}{2} \phi'(F_Y(h(\vec{x}_i^-)))\right),
\end{align*}
where $\lambda$ is a hyper-parameter to be tuned. The TT approach minimizes $\hat{R}_\CDF$ to learn a model. However, the learned model is, again, not always consistent in terms of original risk $R$. This is because, in rich enough hypothesis space $\mathcal{H}$, the minimizer $h_{\CDF} = F_Y^{-1}\left(\expect{Y|\vec{X} = \vec{x}}{F_Y(Y)}\right)$ of \eqref{eq:bregman-CDF-decomposition} is different from $\expect{Y|\vec{X} = \vec{x}}{Y}$, the minimizer of \eqref{eq:expected_loss}, unless target $Y$ distributes uniformly. Hence, for non-uniform target, we cannot always obtain a consistent model. However, we can still derive an estimation error bound if $h_\CDF \in \mathcal{H}$ and target variable $Y$ is generated as
\begin{align}
    Y = h_\mathrm{true}(\vec{X}) + \varepsilon, \label{eq:target-generation}
\end{align}
where $h_\mathrm{true}: \mathcal{X}\to\mathcal{Y}$ is the true target function and $\varepsilon$ is a zero-mean noise variable bounded in $[-\sigma,\sigma]$ for some constant $\sigma$. 
\begin{thm}\label{thm:general-generalization-bound}
Assume that target variable $Y$ is generated by \eqref{eq:target-generation} and $h_{\CDF} \in \mathcal{H}$.  If the pseudo-dimensions of $\{\vec{x} \to \phi'(F_Y(h(\vec{x})))|h \in \mathcal{H}\}, \{\vec{x} \to \phi'(F_Y(h(\vec{x})))|h \in \mathcal{H}\}$ are finite and there exist constants $P>p>0$ such that $p \leq f_Y(y) \leq P$ for all $y \in \mathcal{Y}$, we have
\begin{align*}
    R(\hat{h}_\CDF) \leq R(h_\mathrm{true}) + \left(\frac{P}{p}\sigma\right)^2 + O\left(\sqrt{\frac{\log 1/\delta}{n_\unlabeled}}\right) + O\left(\sqrt{\frac{\log 1/\delta}{n_\rank}}\right)
\end{align*}
with probability $1-\delta$ for $\phi(x) = x^2$, where $\hat{h}_\CDF$ is the minimizer of risk  $\hat{R}_\CDF$ in $\mathcal{H}$.
\end{thm}

The proof can be found in Appendix~\ref{sec:proof-of-general-generalization-bound}. From Theorem~\ref{thm:general-generalization-bound}, we can see that $\hat{h}_\CDF$ is not necessarily consistent. Again, this is inevitable due to the same reason as the RA approach. By comparing Theorems~\ref{thm:uniform-taylor-bound} and \ref{thm:general-generalization-bound}, we can see that the TT approach is more advantageous than the RA approach when the target contains less noise. In section~\ref{sec:experiments}, we empirically compare these approaches and show that which approach is more suitable differs from case to case.

\section{Experiments} \label{sec:experiments}
In this section, we present the empirical performances of proposed methods in the experiments based on synthetic data and benchmark data. We show that our proposed methods outperform the naive method described in \eqref{eq:naive-regression-function} and have a similar performance to a model learned from ordinary supervised learning with labeled data. All codes are available on Github.

Before presenting the results, we describe the detailed procedure of experiments. In all experiments, we consider $l_2$-loss $l(z,t) = (z-t)^2$, which corresponds to setting $\phi(x) = x^2$ in Bregman divergence $d_\phi(t,z)$. The performance is also evaluated by the mean suqared error (MSE) in the held-out test data. We repeat each experiments for 100 times and report the mean and the standard deviation. We employ hypothesis space with linear functions $\mathcal{H} = \{h(\vec{x}) = \vec{\theta}^\top \vec{x}~|~\vec{\theta}\in\mathbb{R}^d\}$. The procedure of hyper-parameter tuning in $R_\linear$ and $R_\CDF$ can be found in Appendix~\ref{sec:experiments-detals}.

We introduce two types of baseline methods. One is a naive application of the ranking methods described in \eqref{eq:naive-regression-function}, in which we use SVMRank \citep{Herbrich2000} as a ranking method. To have a fair comparison, we use the linear kernel in SVMRank. The other is an ordinary supervised linear regression (LR), in which we fit a linear model using the true labels in unlabeled data $\mathcal{D}_\unlabeled$. Note that LR does not use pairwise comparison data $\mathcal{D}_\rank$.


\paragraph{Result for Synthetic Data.} First, we show the result for the synthetic data, in which we know the true marginal $P_Y$. We sample $5$-dimensional unlabeled data $\mathcal{D}_\unlabeled$ from normal distribution $\mathcal{N}(\vec{0}, I_d)$, where $I_d$ is the identity matrix. Then, we sample true unknown parameter $\vec{\theta}$ such that $\|\vec{\theta}\|_2 = 1$ uniformly at random. Target $Y$ is generated as $Y = \vec{\theta}^\top\vec{X} + \varepsilon$, where $\varepsilon$ is a noise following $\mathcal{N}(0,0.1)$. Consequently, $P_Y$ corresponds to $\mathcal{N}(0, \sqrt{1.01})$, which is utilized in proposed methods and the ranking baseline. The pairwise comparison data is generated by \eqref{eq:def_X_compare}. We first sample two features $\vec{X},\vec{X}'$ from $\mathcal{N}(\vec{0}, I_d)$, and then, compare them based on the target value $Y,Y'$ calculated by $Y = \vec{\theta}^\top\vec{X} + \varepsilon$. We fix $n_\unlabeled$ to 100,000 and alter $n_\rank$ from 20 to 10,240 to see the change of performance with respect to the size of pairwise comparison data.

\begin{figure}[t]
    \centering
    \begin{tabular}{cc}
      \begin{minipage}[t]{0.45\hsize}
        \centering
        \includegraphics[width=\columnwidth]{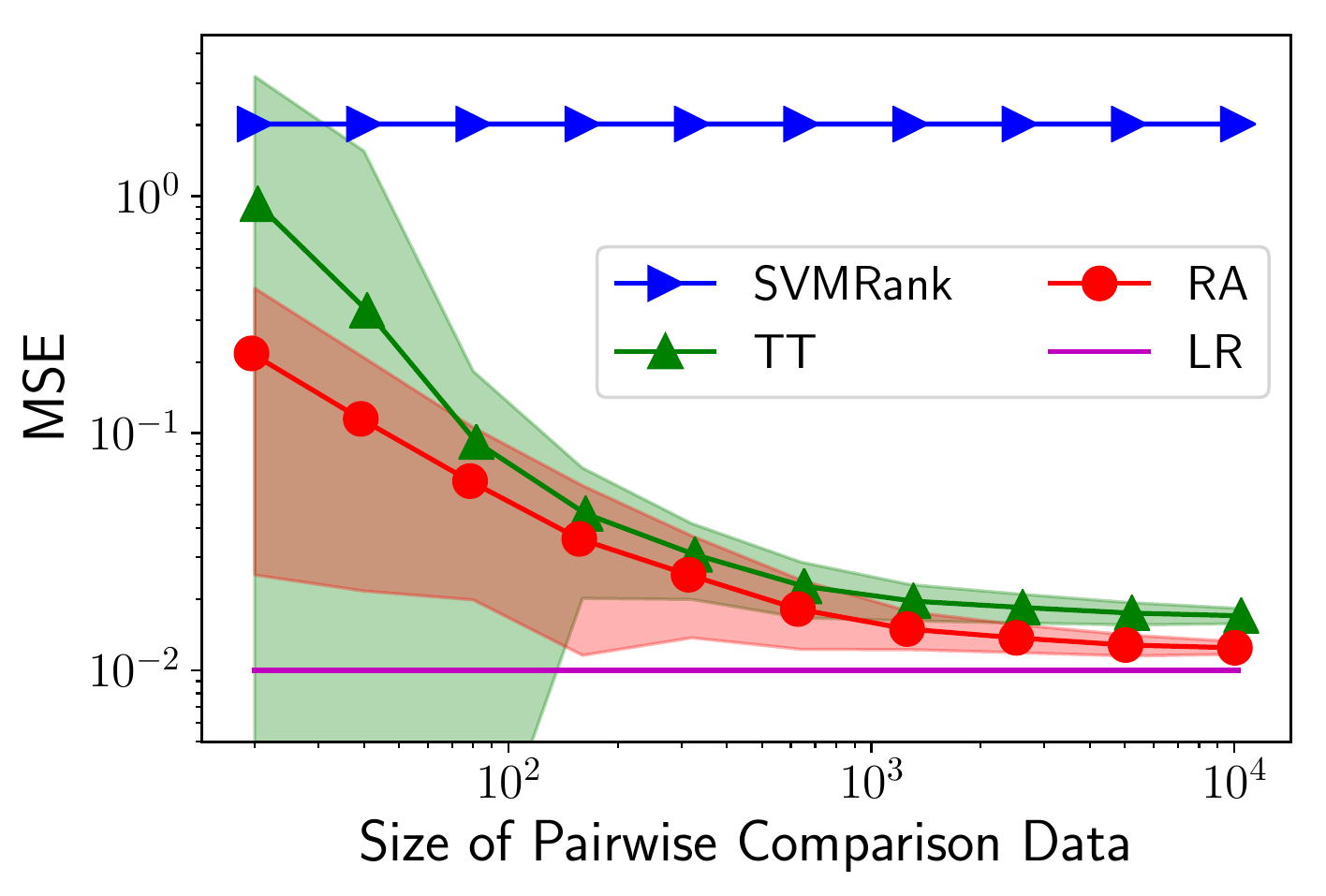}
        \vspace{-20pt}
        \caption{MSE for Synthetic Data}
        \label{fig:normal_dist}
      \end{minipage} &
      \begin{minipage}[t]{0.45\hsize}
        \centering
        \includegraphics[width=\columnwidth]{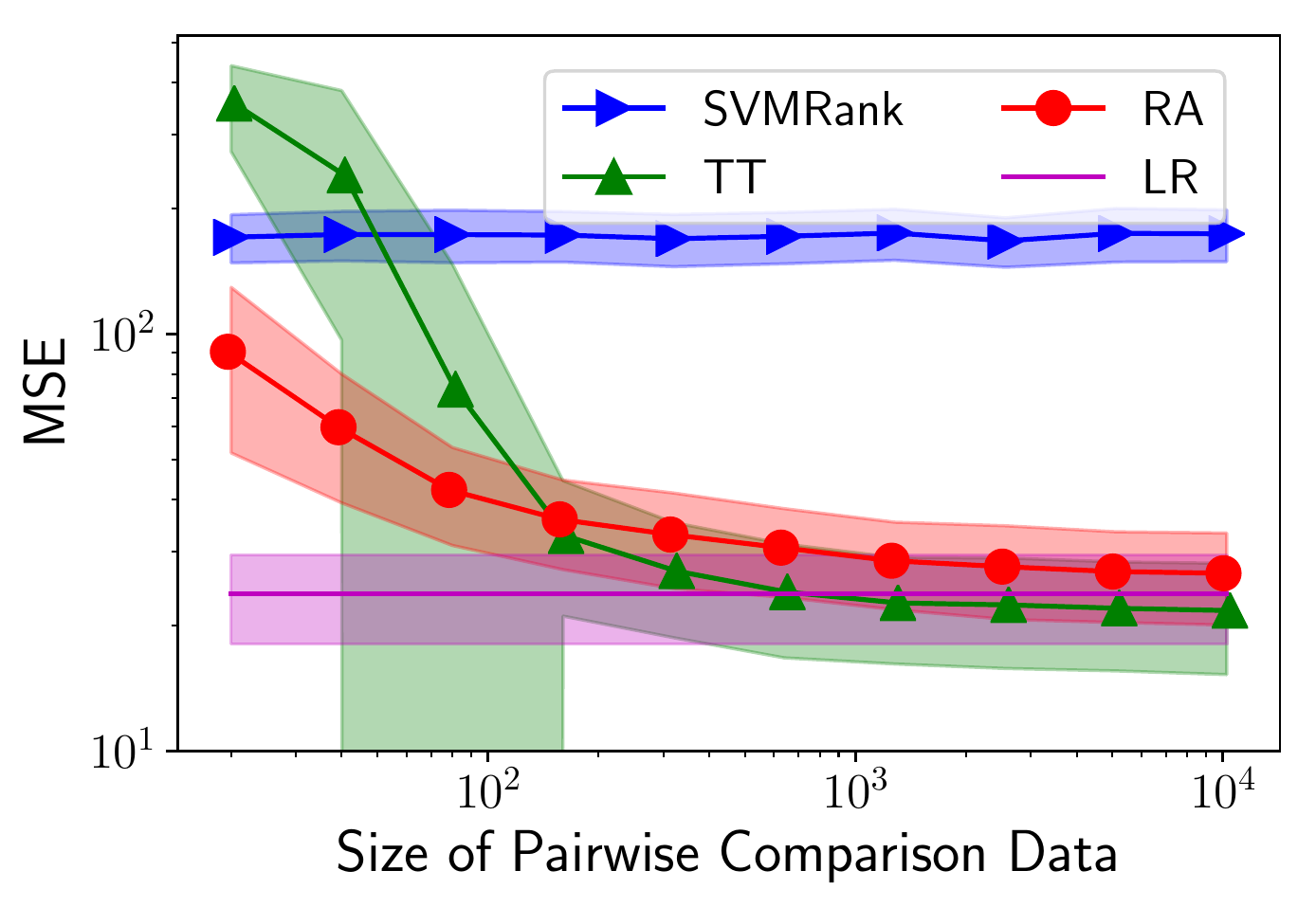}
        \vspace{-20pt}
        \caption{MSE for {\tt housing} Dataset}
        \label{fig:boston}
      \end{minipage}
    \end{tabular}
\end{figure}

The result is presented in Figure~\ref{fig:normal_dist}. From this figure, we can see that with sufficient pairwise comparison data, the performances of our methods are significantly better than SVMRank baseline and close to LR. This is astonishing since LR uses the true label of $\mathcal{D}_\unlabeled$, while our methods do not.

Moreover, we can see that the TT approach outperforms the RA approach with sufficient pairwise
comparison data.
This observation can be understood from the estimation error bound in Theorem~\ref{thm:uniform-taylor-bound}, where the term $\mathrm{Err}(w_1,w_2)$ becomes dominant when sufficient data is provided.
This term $\mathrm{Err}(w_1,w_2)$ becomes large in this synthetic data since $Y$ is not bounded.
Hence, the guarantee of the RA approach becomes weaker than the TT approach when $n_\rank$ is large 
enough, which results in the inferior empirical performance of the RA approach.

Meanwhile, when the size of pairwise comparison data is small, the TT approach is unstable and worse than the RA approach. This is because we learn the quantile value when we minimize $R_\CDF$, and this can be severely inaccurate when the size of pairwise comparison data is small. On the other hand, $R_\linear$ directly minimizes the approximation of true risk $R$, which is less sensitive to small $\mathcal{D}_\rank$.


\paragraph{Result for Benchmark Datasets.} We conducted the experiments for the benchmark datasets as well, in which we do not know true marginal $P_Y$. The details of benchmark datasets can be found in Appendix~\ref{sec:experiments-detals}.  We use the original features as unlabeled data $\mathcal{D}_\unlabeled$. Density function $f_Y$ is estimated from target values in the dataset by kernel density estimation \citep{Parzen1962} with Gaussian kernel. Here, the bandwidth of Gaussian kernel is determined by the cross-validation. The pairwise comparison data is constructed by comparing the true target values of two data points uniformly sampled from $\mathcal{D}_\unlabeled$.

\begin{table}[t]
\caption{MSE for benchmark datasets when $n_\rank$ is 5,000. The bold face means the outstanding method in uncoupled regression methods (SVMRank, RA and TT) chosen by Welch t-test with the significance level $5\%$. Note that LR does not solve uncoupled regression since it uses labels in $\mathcal{D}_\unlabeled$.}
\label{table:sufficient_res}
\centering
\begin{tabular}[tbh]{@{}c@{}cccc@{}} 
\toprule
&Supervised Regression&\multicolumn{3}{c}{Uncoupled Regression}\\
\cmidrule(l{.75em}r{.75em}){2-2}\cmidrule(l{.75em}r{.75em}){3-5}
Dataset  & LR & SVMRank & RA& TT  \\
\midrule
\texttt{ housing }&24.5(5.0)&110.3(29.5)&29.5(6.9)&\textbf{22.5(6.2)}\\
\texttt{ diabetes }&3041.9(219.8)&8575.9(883.1)&\textbf{3087.3(256.3)}&\textbf{3127.3(278.8)}\\
\texttt{ airfoil }&23.3(2.2)&62.1(7.6)&23.7(2.0)&\textbf{22.7(2.2)}\\
\texttt{ concrete }&109.5(13.3)&322.9(45.8)&\textbf{111.7(13.2)}&139.1(17.9)\\
\texttt{ powerplant }&20.6(0.9)&372.2(34.8)&\textbf{21.8(1.1)}&\textbf{22.0(1.0)}\\
\texttt{ mpg }&12.1(2.04)&125(15.1)&12.8(2.16)&\textbf{10.3(2.08)}\\
\texttt{ redwine }&0.412(0.0361)&1.28(0.112)&\textbf{0.442(0.0473)}&0.466(0.0412)\\
\texttt{ whitewine }&0.574(0.0325)&1.58(0.0691)&\textbf{0.597(0.0382)}&0.644(0.0414)\\
\texttt{ abalone }&5.05(0.375)&20.9(1.44)&\textbf{5.26(0.372)}&5.54(0.424)\\
\bottomrule
\end{tabular}
\vskip-10pt
\end{table}

Figure~\ref{fig:boston} shows the performance of each method with respect to the size of pairwise comparison data for {\tt housing} dataset. Although the TT approach performs unstably when $n_\rank$ is small, proposed methods significantly outperform SVMRank and approaches to LR. This fact suggests that the estimation error in $f_Y$ has little impact on the performance. The results for various datasets when $n_\rank$ is 5,000 are presented in Table~\ref{table:sufficient_res}, in which both proposed methods show the promising performances. Note that the approach with less MSE differs by each dataset, which means that we cannot easily judge which approach is better.

\section{Conclusions} \label{sec:conclusion}
In this paper, we proposed novel methods to deal with uncoupled regression by utilizing pairwise comparison data. We introduced two methods, the RA approach and the TT approach, for the problem. The RA approach is to approximate the expected Bregman divergence by the linear combination of expectations of given data, and the TT approach is to learn a model for quantile values and uses the inverse of the CDF to predict the target. We derived estimation error bounds for each method and showed that the learned model is consistent when the target variable distributes uniformly. Furthermore, the empirical evaluations based on both synthetic data and benchmark datasets suggested the competence of our methods. The empirical result also indicated the instability of the TT approach when the size of pairwise comparison data is small, and we may need some regularization scheme to prevent it, which is left for future work.

\subsubsection*{Acknowledge}
LX utilized the facility provided by Masason Foundation. MS was supported by JST CREST Grant Number JPMJCR18A2.

\bibliography{reference}
\newpage
\appendix
\section{Experiments Details} \label{sec:experiments-detals}
In this appendix, we explain the detailed setting of experiments. First, we describe the procedure of the hyper-parameter tuning during the experiments. Then, we provide the detail information of benchmark datasets.

\subsection{Procedure of hyper-parameter tuning}
To construct risk $\hat{R}_\linear$, we need to tune $\lambda, w_1, w_2$, which is done by minimizing empirically approximated $\mathrm{Err}(w_1,w_2)$ defined in \eqref{eq:err-def}. Let $\overline{y},\underline{y}$ be the $0.99$-quantile and $0.01$-quantile of $P_Y$, respectively. Note that we can calculate these quantities since we have access to $f_Y$. Then, we define $\{y^{(i)}\}_{i=1}^{n_{\mathrm{split}}+1}$ as $y_i = \underline{y} + (i-1)/n_{\mathrm{split}}(\overline{y}-\underline{y})$, by which $\mathrm{Err}(w_1, w_2)$ is approximated as 
\begin{align*}
    \mathrm{Err}(w_1, w_2) \simeq \sum_{i=1}^{n_{\mathrm{split}}+1} f_Y(y_i) |y_i - w_1F_Y(y_i) - w_2(1-F_Y(y_i))|.
\end{align*}
We employ $w_1,w_2$ that minimize the empirical approximation above with $n_{\mathrm{split}}=1000$ and fix $\lambda$ to be $(w_1+w_2)/2$ in all cases.

We also use approximation in $R^\lambda_\CDF$ in order to reduce the computational time. Instead of calculating $F_Y(h(\vec{x}))$, we use $\sigma(h(\vec{x}))$, where $\sigma$ is logistic function $\sigma(x) = 1/(1+\exp(-x))$. We fix $\lambda = 1/2$ for this risk, and what we have minimized during the experiments is
\begin{align*}
    R_{\text{TT--emp}}(h) &= \mathfrak{C} - \frac{1}{n_\unlabeled} \sum_{\vec{x}_i \in \mathcal{D}_\unlabeled} \left(\frac12 -\sigma(h(\vec{x}_i))\right)\phi'(\sigma(h(\vec{x}_i))) + \phi(\sigma(h(\vec{x}_i))) \\
    &\quad~~~~~  - \frac{1}{n_\rank}\sum_{(\vec{x}^+_i, \vec{x}^-_i)\in \mathcal{D}_\rank} \frac{1}{4}\phi'(\sigma(h(\vec{x}_i^+))) - \frac{1}{4} \phi'(\sigma(h(\vec{x}_i^-))).
\end{align*}
After obtaining the minimizer $\tilde{h}_\CDF$ of $\tilde{R}_\CDF$, we predict the target by $F_Y^{-1}(\sigma(\tilde{h}_\CDF))$.

\subsection{Benchmark dataset details}
We use eight benchmark datasets from UCI repository \citep{Dua2019} and one (\texttt{diabetes}) from \citet{Efron04}. The details of datasets can be found in Table~\ref{table:datasets}. As preprocessing, we excluded all instances contains missing value, and we encoded categorical feature in \texttt{abalone} as one-hot vector.

\begin{table}[t]
\caption{Information of benchmark datasets.}
\label{table:datasets}
\centering
\begin{tabular}[tbh]{@{}cccc@{}} 
\toprule
Dataset  & Datasize & $d$ & Source \\
\midrule
\texttt{ housing }&404&13&UCI Repository\\
\texttt{ diabetes }&353&10&\citep{Efron04}\\
\texttt{ airfoil }&1202&5&UCI Repository\\
\texttt{ concrete }&824&8&UCI Repository\\
\texttt{ powerplant }&7654&4&UCI Repository\\
\texttt{ mpg }&313&7&UCI Repository\\
\texttt{ redwine }&1279&11&UCI Repository\\
\texttt{ whitewine }&3918&11&UCI Repository\\
\texttt{ abalone }&3341&10&UCI Repository\\
\bottomrule
\end{tabular}
\vskip-10pt
\end{table}
\section{Estimating Density Function and Cumulative Distribution Function} \label{sec:est-prob}
In this section, we discuss the case where the true probability density function $f_Y$ is not given. In such a case, we need a slight modification of proposed approaches since we have to estimate $f_Y$ from the set of target values $\mathcal{D}_Y = \{y_i\}_{i=1}^{n_Y}$, where $n_Y$ is the size of $\mathcal{D}_Y$. We first introduce modification of the RA approach and derive a estimation error bound for it. Then, we discuss the same for the TT approach as well.

\subsection{Modification of the risk approximation approach}

Although $\hat{R}_\linear$ does not depend on $f_Y$ or $F_Y$, we need the information of $P_Y$ when tuning weights $w_1,w_2$, which is done by the minimization of $\mathrm{Err}$ defined in \eqref{eq:err-def}. Since, $\mathrm{Err}$ can not be directly calculated without $f_Y$ and $F_Y$, we propose another quantity $\widehat{\mathrm{Err}}$ below, which substitute expectation over $P_Y$ and CDF function $F_Y$ to empirical mean and the empirical CDF.  
\begin{align*}
    \widehat{\mathrm{Err}}(w_1,w_2) = \frac{1}{n_Y} \sum_{i=1}^{n_Y} |y_i - w_1\hat{F}_Y(y_i) - w_2(1-\hat{F}_Y(y_i))|,
\end{align*}
where $\hat{F}_Y$ is the empirical CDF defined as
\begin{align*}
    \hat{F}_Y(y) = \frac{1}{n_Y} \sum_{i=1}^{n_Y} \indi{y_i \leq y}.
\end{align*}
Note that $\widehat{\mathrm{Err}}$ can be minimized given $\mathcal{D}_Y$. To show the validity of the method, we establish an estimation error bound involving $\widehat{\mathrm{Err}}$ as follows.

\begin{thm}\label{thm:est_prob_RA}
Let $\mathcal{Y}$ be bounded in $\mathcal{Y} \subseteq [-L,L]$. Then, for all $w_1,w_2 \in [-L,L]$, we have
\begin{align*}
    |\mathrm{Err}(w_1,w_2)-\widehat{\mathrm{Err}}(w_1,w_2)| \leq O\left(\sqrt{\frac{\log \delta}{n_Y}}\right)
\end{align*}
with probability $1-2\delta$.
\end{thm}

\begin{proof}
Since the weights are bounded, from \citet[Thm. 10.3]{Mohri2012}, we have
\begin{align*}
    \mathrm{Err}(w_1,w_2) \leq \frac{1}{n_Y} \sum_{i=1}^{n_Y} |y_i - w_1F_Y(y_i) - w_2(1-F_Y(y_i))| +O\left(\sqrt{\frac{\log 1/\delta}{m}}\right),
\end{align*}
with probability $1-\delta$. Furthermore, from Dvoretzky-Kiefer-Wolfowitz Inequality \citep{Massart1990}, we have
\begin{align}
    \|F_Y(y) - \hat{F}_Y(y)\|_\infty \leq \sqrt{\frac{\log (2/\delta)}{2n_Y}} \label{eq:dkw}
\end{align}
with probability $1-\delta$, which yields
\begin{align*}
    \frac{1}{n_Y} \sum_{i=1}^{n_Y} |y_i - w_1F_Y(y_i) - w_2(1-F_Y(y_i))| \leq \widehat{\mathrm{Err}} +O\left(\sqrt{\frac{\log 1/\delta}{m}}\right).
\end{align*}
Therefore, from the union bound, we have
\begin{align*}
    |\mathrm{Err}(w_1,w_2)-\widehat{\mathrm{Err}}(w_1,w_2)| \leq O\left(\sqrt{\frac{\log \delta}{n_Y}}\right)
\end{align*}
with probability $1-2\delta$.
\end{proof}

From Theorems~\ref{thm:uniform-taylor-bound} and \ref{thm:est_prob_RA}, we have
\begin{align*}
    R(\hat{h}_\linear) \leq R(h^*) + O\left(\sqrt{\frac{\log 1/\delta}{n_\unlabeled}}\right)
    + O\left(\sqrt{\frac{\log 1/\delta}{n_\rank}}\right)
    + O\left(\sqrt{\frac{\log 1/\delta}{n_Y}}\right) + M\widehat{\mathrm{Err}}(w_1,w_2),
\end{align*}
with probability $1-5\delta$ under the conditions given in these theorems.

\subsection{Modification on the target transformation approach}

On the other hand, we have $F_Y$ in risk $\hat{R}_\CDF$. Let $\tilde R_\CDF$ be the risk which substitute $F_Y$ in $R_\CDF$ to empirical CDF, defined as
\begin{align*}
    \tilde{R}_{\CDF}(h; \lambda) &= \mathfrak{C} - \frac{1}{n_\unlabeled} \sum_{\vec{x}_i \in \mathcal{D}_\unlabeled} \left((\lambda -\hat{F}_Y(h(\vec{x}_i)))\phi'(\hat{F}_Y(h(\vec{x}_i))) + \phi(\hat{F}_Y(h(\vec{x}_i)))\right) \notag\\
    &\quad~~~~~  - \frac{1}{n_\rank}\sum_{(\vec{x}^+_i, \vec{x}^-_i)\in \mathcal{D}_\rank} \left(\frac{1-\lambda}{2}\phi'(\hat{F}_Y(h(\vec{x}_i^+))) - \frac{\lambda}{2} \phi'(\hat{F}_Y(h(\vec{x}_i^-))),\right).
\end{align*}
Using \eqref{eq:dkw}, we have
\begin{align*}
    |\hat{R}_{\CDF}(h) - \tilde{R}_{\CDF}(h)| \leq O\left(\sqrt{\frac{\log 1/\delta}{n_Y}}\right)
\end{align*}
for all $h\in\mathcal{H}$ with probability $1-\delta$. Let $\tilde{h}_\CDF$ be the minimizer of $\tilde{R}_\CDF$ in hypothesis space $\mathcal{H}$. Then, under the condition given in Theorem~\ref{thm:general-generalization-bound}, we have
\begin{align*}
    R_\CDF(\tilde{h}_\CDF) \leq R_\CDF(h_\CDF) + O\left(\sqrt{\frac{\log 1/\delta}{n_Y}}\right) + O\left(\sqrt{\frac{\log 1/\delta}{n_\rank}}\right) + O\left(\sqrt{\frac{\log 1/\delta}{n_\unlabeled}}\right),
\end{align*}
with probability $1-4\delta$, therefore we have
\begin{align*}
    R(\tilde{h}_\CDF) \leq R(h^*) +2\left(\frac{P}{p}\sigma\right)^2 + O\left(\sqrt{\frac{\log 1/\delta}{n_Y}}\right) + O\left(\sqrt{\frac{\log 1/\delta}{n_\rank}}\right) + O\left(\sqrt{\frac{\log 1/\delta}{n_\unlabeled}}\right),
\end{align*}
with probability $1-4\delta$, which can be shown by the slight modification of the proof of Theorem~\ref{thm:general-generalization-bound}.

\section{Proofs} \label{sec:proofs}

\subsection{Proof of Lemma~\ref{lem:X_compare_distribution}} \label{sec:proof-of-X_compare_distribution}

Lemma~\ref{lem:X_compare_distribution} can be proved as follows.
\begin{proof}[Proof of Lemma~\ref{lem:X_compare_distribution}]
    Let $f_{\vec{X}^+}$ be the probability density function (PDF) of $P_{\vec{X}^+}$. From the definition of $\vec{X}^+$, we have
    \begin{align*}
        f_{\vec{X}^+}(\vec{x}) &= \frac{1}{Z}\iiint f_{\vec{X},Y}(\vec{x}, y)f_{\vec{X},Y}(\vec{x}', y')\indi{y > y'} \intd y\intd y' \intd \vec{x}'\\
        &= \frac{1}{Z}\int f_{\vec{X},Y}(\vec{x}, y) \left[\int f_Y(y')\indi{y > y'}\intd y'\right] \intd y \\
        &= \frac{1}{Z}\int f_{\vec{X},Y}(\vec{x}, y) F_Y(y) \intd y,
    \end{align*}
    where $Z$ is the normalizing constant and $f_{\vec{X},Y}(y)$ is the PDF of $P_{\vec{X},Y}$. Now, $Z$ is calculated as
    \begin{align*}
        Z &= \iint f_{\vec{X},Y}(\vec{x}, y) F_Y(y) \intd y \intd \vec{x}\\
          &= \int f_Y(y) F_Y(y) \intd y\\
          &= \frac12.
    \end{align*}
    The last equality holds from the integration by parts. Therefore, we have
    \begin{align*}
        \expect{\vec{X}^+}{\phi'(\vec{X}^+)} &= \int f_{\vec{X}^+}(\vec{x}) \phi'(\vec{x})\intd \vec{x}\\
        &= \int 2 \left\{\int f_{\vec{X},Y}(\vec{x}, y) F_Y(y) \intd y\right\} \phi'(\vec{x})\intd \vec{x}\\
        &= \expect{\vec{X},Y}{F_Y(Y)\phi'(\vec{x})}.
    \end{align*}
    The expectation over $P_{\vec{X}^-}$ can be derived in the same way.
\end{proof}

\subsection{Proof of Theorem~\ref{thm:uniform-taylor-bound}} \label{sec:proof-of-uniform-taylor-bound}

Here, we show the proof of Theorem~\ref{thm:uniform-taylor-bound}. First, we show the gap between $R$ and $R_\linear$ can be bounded as follows.

\begin{lem}\label{lem:linear-loss-difference}
For all $h \in \mathcal{H}$, such that  $|\phi'(h(\vec{x}))| \leq M$ for all $\vec{x} \in \mathcal{X}$, we have
\begin{align*}
    |R(h) - R_\linear(h; \lambda; w_1,w_2)| \leq M\mathrm{Err}(w_1,w_2)
\end{align*}
for all $\lambda\in\mathbb{R}$.
\end{lem}
\begin{proof}
    From Lemma~\ref{lem:X_compare_distribution} and the fact $\expect{\vec{X}}{\phi'(\vec{X})} = \frac12\expect{\vec{X}^+}{\phi'(\vec{X}^+)} + \frac12\expect{\vec{X}^-}{\phi'(\vec{X}^-)}$, we have
    \begin{align*}
        &|R(h) - R_\linear(h;\lambda, w_1,w_2)| \\
        &\quad = \left|\expect{\vec{X},Y}{Y\phi'(h(\vec{X}))} - w_1\expect{\vec{X}^+}{\phi'(h(\vec{X}^+))} - w_2\expect{\vec{X}^-}{\phi'(h(\vec{X}^-))}\right| \\
        &\quad= \left|\int f_{\vec{X},Y}(\vec{x},y) \phi'(h(\vec{x}))\{y-2w_1F_Y(y)-2w_2(1-F_Y(y))\}\intd y \intd \vec{x} \right| \\
        &\quad\leq \int f_{\vec{X},Y}(\vec{x},y) \left|\phi'(h(\vec{x}))\right|\left|y-2w_1F_Y(y)-2w_2(1-F_Y(y))\right|\intd y \intd \vec{x} \\
        &\quad\leq M \int f_{Y}(y)\left|y-2w_1F_Y(y)-2w_2(1-F_Y(y))\right|\intd y \\
        &\quad\leq M \mathrm{Err}(w_1,w_2).
    \end{align*}
\end{proof}

Now, Theorem~\ref{thm:uniform-taylor-bound} can be derived as follows.
\begin{proof}[Proof of Theorem~\ref{thm:uniform-taylor-bound}]
    Let $\tilde d, \tilde d'$ be the pseudo-dimensions defined as
\begin{align*}
    &\tilde d = \mathrm{Pdim}(\{\vec{x} \to \phi'(h(\vec{x})) ~|~h \in \mathcal{H}\}),\\
    &\tilde d' = \mathrm{Pdim}(\{\vec{x} \to h(\vec{x})\phi'(h(\vec{x})) - \phi(h(\vec{x}))~|~h \in \mathcal{H}\}),
\end{align*}
where $\mathrm{Pdim}(\mathcal{F})$ denotes the pseudo-dimension of the functional space $\mathcal{F}$.
    From the assumptions in Theorem~\ref{thm:uniform-taylor-bound}, using the discussion in \citet[Theorem 10.6]{Mohri2012}, each of following bound holds with probability $1-\delta$ for all $h \in \mathcal{H}$.
     \begin{align*}
         &\left|\expect{\vec{X}^+}{\phi'(h(\vec{X}^+))}-\frac{1}{n_\rank} \sum_{\vec{x}_i^+ \in \mathcal{D}^+_\rank} \phi'(h(\vec{x}_i^+)) \right| \leq M\sqrt{\frac{2\tilde{d}\log \frac{\mathrm{e}n_\rank}{\tilde{d}}}{n_\rank}} + M\sqrt{\frac{\log \frac{1}{\delta}}{2n_\rank}},\\
         &\left|\expect{\vec{X}^-}{\phi'(h(\vec{X}^-))}-\frac{1}{n_\rank} \sum_{\vec{x}_i^- \in \mathcal{D}^-_\rank} \phi'(h(\vec{x}_i^-)) \right| \leq M\sqrt{\frac{2\tilde{d}\log \frac{\mathrm{e}n_\rank}{\tilde{d}}}{n_\rank}} + M\sqrt{\frac{\log \frac{1}{\delta}}{2n_\rank}},\\
         &\left|\expect{\vec{X}}{g(\vec{X})}-\frac{1}{n_\unlabeled} \sum_{\vec{x}_i \in \mathcal{D}^+_\unlabeled} g(\vec{x}_i) \right|  \leq m\sqrt{\frac{2\tilde{d}'\log \frac{\mathrm{e}n_\unlabeled}{\tilde{d}'}}{n_\unlabeled}} + m\sqrt{\frac{\log \frac{1}{\delta}}{2n_\unlabeled}},
     \end{align*}
     where $g(\vec{x}) = h(\vec{x})\phi'(h(\vec{x})) + \phi(h(\vec{x}))$. From the uniform bound, we have
     \begin{align*}
        &|R_\linear(h; w_1, w_2) - \hat{R}_\linear(h;\lambda, w_1, w_2)|\\
        &\quad \leq \left(\left|w_1-\frac\lambda 2\right|+\left|w_2-\frac\lambda 2\right|\right)\left(M\sqrt{\frac{2\tilde d\log \frac{\mathrm{e}n_\rank}{\tilde d}}{n_\rank}} + M\sqrt{\frac{\log \frac{1}{\delta}}{2n_\rank}}\right)\\
        &\quad~~~~~~~~~~+(m+\lambda M)\left(\sqrt{\frac{2\tilde{d}'\log \frac{\mathrm{e}n_\unlabeled}{\tilde{d}'}}{n_\unlabeled}} +  \sqrt{\frac{\log \frac{1}{\delta}}{2n_\unlabeled}}\right)
     \end{align*}
     with probability $1-3\delta$ for all $h\in\mathcal{H}$. Hence, with probability $1-3\delta$, we have
     \begin{align*}
         &R(\hat{h}_\linear) - R(h^*)\\
         & \quad \leq R_\linear(\hat{h}_\linear;\lambda,w_1,w_2) - R_\linear(h^*;\lambda,w_1,w_2) + |R(h^*) - R_\linear(h^*;\lambda,w_1,w_2)|\\
         &\quad~~~~~ + |R(\hat{h}_\linear) - R_\linear(\hat{h}_\linear;\lambda, w_1,w_2)| \\
         &\quad \leq (R_\linear(\hat{h}_\linear;\lambda,w_1,w_2)- \hat{R}_\linear(h^*;\lambda,w_1,w_2)) \\
         &\quad~~~~~  - (R_\linear(h^*;\lambda,w_1,w_2) - \hat{R}_\linear(h^*;\lambda,w_1,w_2)) +2M\mathrm{Err}(w_1,w_2)\\
         &\quad \leq (R_\linear(\hat{h}_\linear;\lambda,w_1,w_2)- \hat{R}_\linear(\hat{h}_\linear;\lambda,w_1,w_2)) \\
         &\quad~~~~~  - (R_\linear(h^*;,\lambda,w_1,w_2) - \hat{R}_\linear(h^*;\lambda,w_1,w_2)) +2M\mathrm{Err}(w_1,w_2)\\
         &\quad\leq O\left(\sqrt{\frac{\log 1/\delta}{n_\unlabeled}}\right) + O\left(\sqrt{\frac{\log 1/\delta}{n_\rank}}\right)+2M\mathrm{Err}(w_1,w_2),
     \end{align*}
     where the second inequality holds from the fact $\hat{R}_\linear(\hat{h}_\linear;\lambda,w_1,w_2) \leq \hat{R}_\linear(\hat{h}^*;\lambda,w_1,w_2)$ and Lemma~\ref{lem:linear-loss-difference}.
\end{proof}

\subsection{Proof of Theorem~\ref{thm:optimal-lambda-in-uniform}} \label{sec:proof-of-optimal-lambda-in-uniform}
Theorem~\ref{thm:optimal-lambda-in-uniform} can be shown as follows.
\begin{proof}[Proof of Theorem~\ref{thm:optimal-lambda-in-uniform}]
     The variance of $\hat{R}_\linear$ denoted as $\Var{}{\hat{R}_\linear(h;\lambda,w_1,w_2)}$ can be expressed as
    \begin{align*}
        \Var{}{\hat{R}_\linear(h;\lambda,w_1,w_2)} = \left(w_1-\frac{\lambda}{2}\right)^2\frac{\sigma^2_+}{n_\rank} + \left(w_2-\frac{\lambda}{2}\right)^2\frac{\sigma^2_-}{n_\rank}
    \end{align*}
    when $n_\unlabeled \to \infty$.  By solving the above quadratic optimization problem, we have
    \begin{align*}
        \argmin_\lambda \Var{}{\hat{R}_\linear(h; \lambda,w_1,w_2)} =  \frac{2(w_1\sigma^2_+ +w_2\sigma^2_-)}{\sigma^2_+ + \sigma^2_-}.
    \end{align*}
\end{proof}

\subsection{Proof of Theorem~\ref{thm:impossibility}}

We can construct a simple example satisfies the conditions in Theorem~\ref{thm:impossibility} as follows.
\begin{proof}
    Let $f_{\vec{X},Y},\tilde{f}_{\vec{X},Y}$ be the PDF of $P_{\vec{X},Y}, \tilde{P}_{\vec{X},Y}$, respectively. If we consider $\mathcal{X} = [-1,1]$ and $\mathcal{Y} = [0,4]$ and these PDF to be
    \begin{align*}
        &f_{\vec{X},Y}(x,y) = \begin{cases}
            \frac{1}{6} & (y \in [0, 2] \cup [3, 4]),\\
            0 & (\mathrm{otherwise}),
        \end{cases}\\
        &\tilde{f}_{\vec{X},Y}(x,y) = \begin{cases}
            \frac{1}{8} &(x \in [-1,0), y \in [0, 1)),\\
            \frac{1}{4} &(x \in [-1,0), y \in [1, 2)),\\
            \frac{1}{8} &(x \in [-1,0), y \in [3, 4]),\\
            \frac{5}{24} &(x \in [0,1], y \in [0, 1)),\\
            \frac{1}{12} &(x \in [0,1], y \in [1, 2)),\\
            \frac{5}{24} &(x \in [0,1], y \in [3, 4]),\\
            0 & (\mathrm{otherwise}).
        \end{cases}
    \end{align*}
    Then, by the simple calculation, we can see that they have the same PDF $f_{\vec{X}}(x), f_Y(y), f_{\vec{X}^+,\vec{X}^-}(x^+,x^-)$, each represents the PDF of $P_{\vec{X}}, P_Y, P_{\vec{X}^+,\vec{X}^-}$, respectively, which are
    \begin{align*}
        &f_{\vec{X}}(x) = 0.5,\\
        &f_Y(y) = \begin{cases}
            \frac{1}{3} & (y \in [0, 2] \cup [3, 4]),\\
            0 & (\mathrm{otherwise}),
        \end{cases}\\
        &f_{\vec{X}^+,\vec{X}^-}(x^+,x^-) = 0.25.
    \end{align*}
    However, the conditional expectation $\expect{Y|\vec{X}=x}{Y}$ defined on $P_{\vec{X},Y}$ is
    \begin{align*}
        \expect{Y|\vec{X}=x}{Y} = \frac{11}{6},
    \end{align*}
    while the conditional expectation $\tilde{\mathbb{E}}_{Y|\vec{X}=x}[Y]$ defined on $\tilde{P}_{\vec{X},Y}$ is
    \begin{align*}
        \tilde{\mathbb{E}}_{Y|\vec{X}=x}[Y] = \begin{cases}
        \frac{7}{4}& (x \in [-1,0)),\\
        \frac{23}{12}& (x \in [0,1]).\\
        \end{cases}
    \end{align*}
\end{proof}

\subsection{Proof of Theorem~\ref{thm:general-generalization-bound}} \label{sec:proof-of-general-generalization-bound}

The Theorem~\ref{thm:general-generalization-bound} can be shown as follows.
\begin{proof}[Proof of Theorem~\ref{thm:general-generalization-bound}]
    We first show that under the conditions, we have
    \begin{align*}
        \|h_{\mathrm{true}}(\vec{x}) - h_\CDF(\vec{x})\|_\infty \leq \frac{\sigma P}{p}.
    \end{align*}
     Since $(F_Y(y))' = f_Y(y) \leq P$ and $(F^{-1}_Y(y))' = 1/f_Y(y) \leq 1/p$, $F_Y(y), F^{-1}_Y(y)$ are $P, 1/p$-Lipschitz continuous, respectively. Therefore, we have
    \begin{align*}
        h_\CDF(\vec{x}) &= F_Y^{-1}(\expect{Y|\vec{X}=\vec{x}}{F_Y(Y)})\\
        &= F_Y^{-1}(\expect{\epsilon}{F_Y(h_{\mathrm{true}}(\vec{x}) + \varepsilon)})\\
        &\leq F_Y^{-1}(F_Y(h_{\mathrm{true}}(\vec{x}) + \sigma P)\\
        &\leq h_{\mathrm{true}}(\vec{x}) + \frac{\sigma P}{p}.
    \end{align*}
    for all $\vec{x}\in \mathcal{X}$. With the same discussion, we have $|h_\CDF(\vec{x}) - h_\mathrm{true}(\vec{x})| \leq \frac{\sigma P}{p}$. Therefore, we have
    \begin{align*}
        \|h_{\mathrm{true}}(\vec{x}) - h_\CDF(\vec{x})\|_\infty \leq \frac{\sigma P}{p}.
    \end{align*}
    
    Now, if $\phi(x) = x^2$, which means $R(h) = \expect{\vec{X},Y}{(h(\vec{X})-Y)^2}$, we have
    \begin{align*}
        R(\hat{h}_\CDF)& = \expect{\vec{X},Y}{(\hat{h}_\CDF(\vec{x})-Y)^2} \\
        & = \expect{\vec{X}}{(\hat{h}_\CDF(\vec{X})-h_{\mathrm{true}}(\vec{X}))^2} + \expect{\vec{X},Y}{(h_{\mathrm{true}}(\vec{X})-Y)^2}\\
        &\leq R(h_{\mathrm{true}})+ 2\expect{\vec{X}}{(\hat{h}_\CDF(\vec{x})-h_\CDF(\vec{x}))^2}+ 2\expect{\vec{X}}{(h_{\mathrm{true}}(\vec{x})-h_\CDF(\vec{x}))^2}.
    \end{align*}
    Since $\|h_\CDF(\vec{x}) - h_{\mathrm{true}}(\vec{X})\|_\infty \leq \frac{\sigma P}{p}$, we have
    \begin{align*}
        \expect{\vec{X}}{(h_{\mathrm{true}}(\vec{X})-h_\CDF(\vec{X}))^2} \leq \left(\frac{\sigma P}{p}\right)^2.
    \end{align*}
    Furthermore, using the characteristic of expectation, if $\phi(x) = x^2$, which means $R_\CDF(h) = \expect{\vec{X},Y}{(F_Y(h(\vec{X}))-F_Y(Y))^2}$, we have
    \begin{align*}
        &R_\CDF(\hat{h}_\CDF)\\
        &\quad= \expect{\vec{X},Y}{(F_Y(\hat{h}_\CDF(\vec{X}))-F_Y(Y))^2} \\
        &\quad= \expect{\vec{X},Y}{(F_Y(\hat{h}_\CDF(\vec{X}))-F_Y(h_\CDF(\vec{X})))^2}  + \expect{\vec{X},Y}{(F_Y(Y) - F_Y(h_\CDF(\vec{X})))^2}\\
        &\quad= \expect{\vec{X},Y}{(F_Y(\hat{h}_\CDF(\vec{X}))-F_Y(h_\CDF(\vec{X})))^2} + R_\CDF(h_\CDF).
    \end{align*}
    Since $(F_Y(y))' \geq p$, we have
    \begin{align*}
        \expect{\vec{X}}{(\hat{h}_\CDF(\vec{X}) - h_\CDF(\vec{X}))^2} &\leq \frac{1}{p^2} \expect{\vec{X},Y}{(F_Y(\hat{h}_\CDF(\vec{X}))-F_Y(h_\CDF(\vec{X})))^2}\\
        &= \frac{1}{p^2} \left(R_\CDF(\hat{h}_\CDF) - R_\CDF(h_\CDF)\right)\\
        &\leq O\left(\sqrt{\frac{\log 1/\delta}{n_\unlabeled}}\right) + O\left(\sqrt{\frac{\log 1/\delta}{n_\rank}}\right)
    \end{align*}
    with probability $1-3\delta$, where the last inequality holds from the same discussion as in Theorem~\ref{thm:uniform-taylor-bound}. Note that $|\phi'(F_Y(h(\vec{x})))|, |F_Y(h(\vec{x}))\phi'(F_Y(h(\vec{x})))-\phi(F_Y(h(\vec{x})))|$ are bounded since $F_Y(h(\vec{x})) \in [0,1]$ by definition. Combining these inequalities, we can see that
    \begin{align*}
        R(\hat{h}_\CDF) &\leq R(h_{\mathrm{true}}(\vec{x}))+ 2 \left(\frac{\sigma P}{p}\right)^2 +O\left(\sqrt{\frac{\log 1/\delta}{n_\unlabeled}}\right) + O\left(\sqrt{\frac{\log 1/\delta}{n_\rank}}\right)
    \end{align*}
    with probability $1-3\delta$.
\end{proof}

\end{document}